\documentclass[preprint,12pt]{elsarticle}

\usepackage{amsmath}
\usepackage{amssymb}
\usepackage{amsthm}
\usepackage{hyperref}
\usepackage{cleveref}
\crefformat{footnote}{#2\footnotemark[#1]#3}
\usepackage[algoruled,boxruled, linesnumbered, shortend, lined]{algorithm2e}
\usepackage{algpseudocode}
\usepackage{enumitem}
\usepackage{caption}
\usepackage{graphicx}
\usepackage{float}
\usepackage{epsfig}

\newcommand{\argmin}{\operatornamewithlimits{argmin}}

\newtheorem{theorem}{Theorem}[section]

\newtheorem{lemma}[theorem]{Lemma}

\newtheorem{corollary}[theorem]{Corollary}

\usepackage{lineno}

\journal{XXXX}

\begin{document}

\begin{frontmatter}

%% Title, authors and addresses

%% use the tnoteref command within \title for footnotes;
%% use the tnotetext command for theassociated footnote;
%% use the fnref command within \author or \address for footnotes;
%% use the fntext command for theassociated footnote;
%% use the corref command within \author for corresponding author footnotes;
%% use the cortext command for theassociated footnote;
%% use the ead command for the email address,
%% and the form \ead[url] for the home page:
%% \title{Title\tnoteref{label1}}
%% \tnotetext[label1]{}
%% \author{Name\corref{cor1}\fnref{label2}}
%% \ead{email address}
%% \ead[url]{home page}
%% \fntext[label2]{}
%% \cortext[cor1]{}
%% \address{Address\fnref{label3}}
%% \fntext[label3]{}

\title{Accelerating Physics-Informed Neural Network Training with Prior Dictionaries}

%% use optional labels to link authors explicitly to addresses:
\author[label1]{Wei Peng}
\ead{weipeng0098@126.com}
\author[label1]{Weien Zhou\corref{cor1}\fnref{label2}}
\fntext[label2]{Corresponding Author}
\ead{weienzhou@nudt.edu.cn}
\author[label1]{Jun Zhang}
\ead{mcgrady150318@163.com}
\author[label1]{Wen Yao}
\ead{wendy0782@126.com}
\address[label1]{National Innovation Institute of Defense Technology, Chinese Academy of Military Science, 100071, Beijing.}
%% \address[label2]{}

\author{}

\address{}

\begin{abstract}
		Physics-Informed Neural Networks (PINNs) can be regarded as general-purpose PDE solvers, but it might be slow to train PINNs on particular problems, and there is no theoretical guarantee of corresponding error bounds. In this manuscript, we propose a variant called Prior Dictionary based Physics-Informed Neural Networks (PD-PINNs). Equipped with task-dependent dictionaries,  PD-PINNs enjoy enhanced representation power on the tasks, which helps to capture features provided by dictionaries so that the proposed neural networks can achieve faster convergence in the process of training. In various numerical simulations, compared with existing PINN methods, combining prior dictionaries can significantly enhance  convergence speed. In terms of theory, we obtain the error bounds applicable to PINNs and PD-PINNs for solving elliptic partial differential equations of second order. It is proved that under certain mild conditions, the prediction error made by neural networks can be bounded by expected loss of PDEs and boundary conditions.

\end{abstract}

%%Graphical abstract
%\begin{graphicalabstract}
%\includegraphics{grabs}
%\end{graphicalabstract}

%%Research highlights
%\begin{highlights}
%\item We employed task-dependent basis functions in physics-informed neural networks and presented an improved network structure.
%\item The proposed method is simple, and it is shown to accelerate convergence in training physics-informed neural networks.
%\item  We also proved convergence results theoretically.
%\end{highlights}

\begin{keyword}
%% keywords here, in the form: keyword \sep keyword
prior dictionary\sep accelerated training\sep PINN, error bound\sep elliptic equations
%% PACS codes here, in the form: \PACS code \sep code

%% MSC codes here, in the form: \MSC code \sep code
%% or \MSC[2008] code \sep code (2000 is the default)

\end{keyword}

\end{frontmatter}

%% \linenumbers

%% main text
	\section{Introduction}

As is known, neural networks are widely used to solve various scientific computing problems \cite{lin2018all, ding2001multi,wang2004modeling}. A sequence of recent works has applied neural networks  to solve PDEs successfully \cite{long2019pde,raissi2019physics,lu2019deepxde}. We consider the following partial differential equation for the function $u(\cdot)$:
\begin{align}
\mathcal{L}[u](x)&=q(x), \quad x\in\Omega,\nonumber\\
u(x) &= \tilde u(x), \quad x \in\partial \Omega,\label{pde}
\end{align}
the solution of which can be approximated by a neural network. Based on this motivation, Physics-Informed Neural Networks (PINNs) \cite{raissi2017physics} construct such neural networks penalized by the discrepancy between the right-hand side (RHS) and the left-hand side (LHS) of  problem \eqref{pde}.
To make the so-called physics information be learned by the neural networks, the loss function usually consists of three parts: partial differentiable structure loss (PDE loss), boundary value condition loss (BC loss),  and initial value condition loss (IC loss). The structure of a PINN is shown in Figure \ref{pinn_structure}.
\begin{figure}[htbp]
	\centering
	\includegraphics[width=1\textwidth]{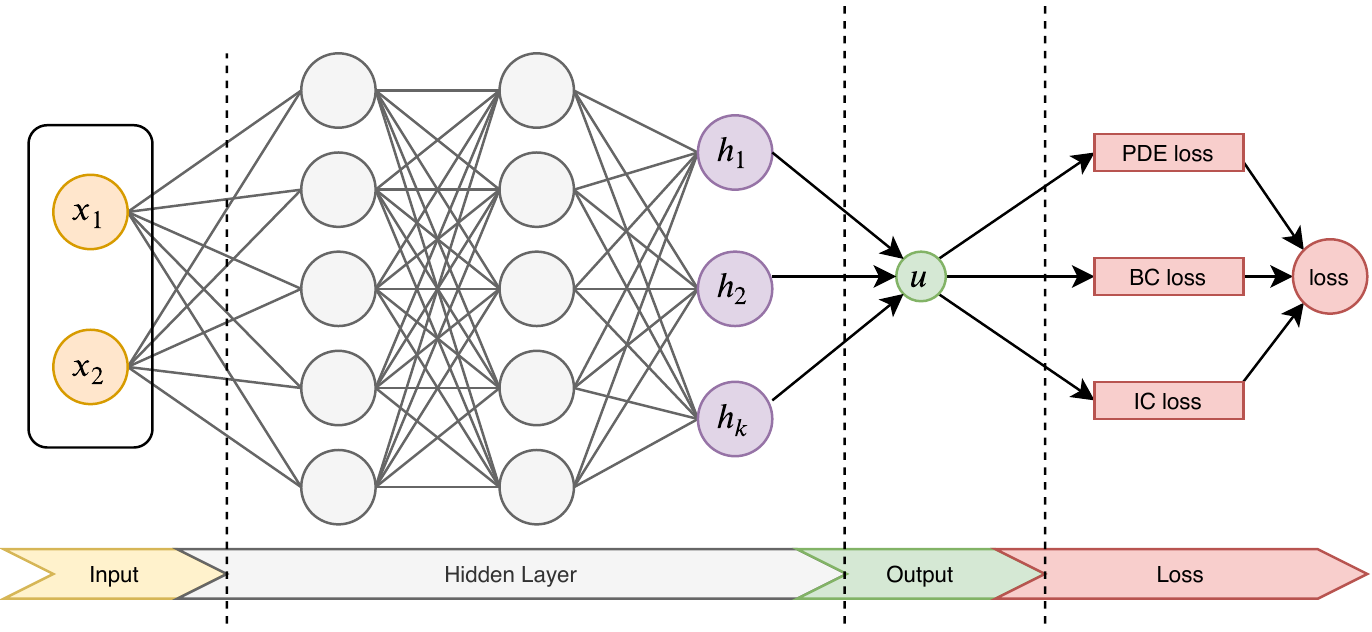}
	\caption{Illustration of a PINN}
	\label{pinn_structure}
\end{figure}

Denote a PINN as F being parameterized by $\Theta$. When there is no ambiguity, we regard IC as BC. Then expected total loss during training consists of two parts:
\begin{align}\label{emprical_loss}
{Loss}(\Theta)={Loss}_{PDE}(\Theta)+{Loss}_{BC}(\Theta).
\end{align}
The expected PDE loss is
\begin{align}\label{pde_loss_1}
{Loss}_{PDE}(\Theta):=\mathbb{E}_{\mathbf{Y}}\left[|\mathcal{L}[F_\Theta](\mathbf{Y})-q(\mathbf{Y})|^2\right],
\end{align}
where random variable $\mathbf{Y}$ is uniformly distributed on $\Omega$. The expected BC loss is
\begin{align}\label{bc_loss_1}
{Loss}_{BC}(\Theta):=\mathbb{E}_{\mathbf{X}}\left[|F_\Theta(\mathbf{X})-\tilde u(\mathbf{X})|^2\right],
\end{align}
where random variable $\mathbf{X}$ is uniformly distributed on $\partial\Omega$.
The optimization problem of training $F_{\Theta}$ is reformulated as follows:
\begin{align}\label{probabilistic_loss}
\bar\Theta \in\argmin_{\Theta}\left\{ \mathbb{E}_{\mathbf{Y}}\left[|\mathcal{L}[F_\Theta](\mathbf{Y})-q(\mathbf{Y})|^2\right]+\mathbb{E}_{\mathbf{X}}\left[|F_\Theta(\mathbf{X})-\tilde u(\mathbf{X})|^2\right]\right\}.
\end{align}
To accelerate finding a minima of \eqref{probabilistic_loss},  \cite{Jagtap2020} considers adjusting general structures of  networks and introduces a trainable variable to scale activation functions adaptively. Later, the subsequent work \cite{jagtap2019locally} uses local adaptive activation functions. From the perspective of data-driven, the idea of leveraging prior structured information is also widely applied to training acceleration, such as wavelet representation\cite{zhang1995wavelet}, periodic structures\cite{rairan2014reconstruction}, symplectic structures\cite{sienko2002hamiltonian}, energy preserving tensors\cite{ling2016reynolds}. These methods employs specially designed neural networks and are applicable to a particular class of problems.

Motivated by these works, in this manuscript we introduce Prior Dictionary based Physics-Informed Neural Networks (PD-PINNs), which integrate prior information into PINNs to accelerate training. As is shown in Figure \ref{PINN-PD}, compared with a PINN, a PD-PINN has an additional dictionary fusion layer, which combines prior information with the output layer of the neural network by the inner product. See section \ref{pinn-pd-section} for detailed description.
\begin{figure}[htbp]
	\centering
	\includegraphics[width=1\textwidth]{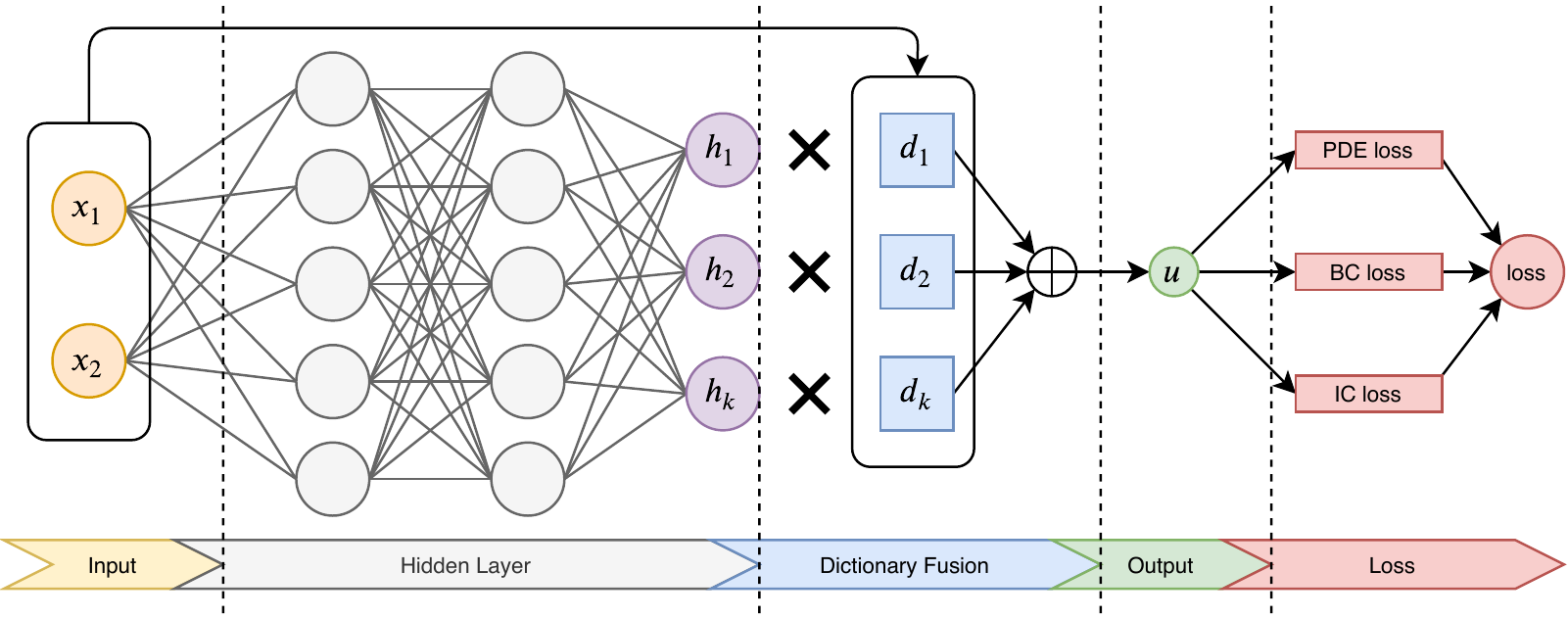}
	\caption{Illustration of a PD-PINN}
	\label{PINN-PD}
\end{figure}
The idea of PD-PINNs is mainly motivated by two aspects. On the one hand, our method is derived from the traditional spectral methods\cite{henrici1979fast,gottlieb1977numerical}, which decompose the ground truth over an orthogonal basis. The methods enjoy the guarantee of spectral convergence, and these basis functions can be regarded as a special type of prior dictionaries.
Nevertheless,  a finite basis expansion would result in truncation error. PD-PINNs utilize the universal approximation ability of neural networks to make up the truncation error, which can achieve high accuracy with a prior dictionary consisting of a small amount of basis functions.  On the other hand, since the essence of training neural networks is to learn representation, we could embed priors into the network before training stages. Therefore, one natural way is to construct a prior dictionary based PINN to achieve the  ``pre-train", thus to accelerate training.

Another issue of PINNs is the suspicious error bounds. There is no guarantee that small PDE loss and BC loss in \eqref{probabilistic_loss} lead to a small total loss. To partially address this problem, we propose an error bound for PINNs solving elliptic PDEs of second order in the sense of $\|\cdot\|_\infty$ under some mild conditions. 

The main contribution of this manuscript is twofold.
\begin{enumerate}
	\item On the other hand, we propose a variant of PINNs, i.e., PD-PINNs, which employ prior information and accelerate training of neural networks. For various PDE problems defined on different types of domains, we construct corresponding prior dictionaries. The numerical simulations illustrate accelerated convergence of PD-PINNs. PD-PINNs can
	even recover the true solutions of some problems where PINNs hardly converge.
	\item On the one hand, we have proved that the error between the neural network and the ground truth in the sense of the infinity norm can be bounded by the two terms on the RHS of \eqref{probabilistic_loss}. Accordingly it is guaranteed that minimizing the RHS of \eqref{emprical_loss} makes the neural network approach the true solution. 
\end{enumerate}
The rest of the manuscript is organized as follows. Section 2 introduces the method and provides the theoretical error bound of PINNs on elliptic PDE of second order.
Four numerical simulations on synthetic problems are conducted in section 3.
Section 4 concludes this manuscript.
\section{Methodology}
We first provide construction details of PD-PINNs in subection \ref{pinn-pd-section}. 
Then we provide theoretical guarantees for PINN error bounds in section \ref{pinn-thm}.
\subsection{PD-PINN}\label{pinn-pd-section}
Let $N_\Theta(\cdot)$ be a neural network parameterized by $\Theta$, which is a mapping from $\mathbb{R}^d$ into $\mathbb{R}^N$.
We employ a Multi-Layer Perceptron (MLP) with the activation function $\sigma(\cdot):=\tanh(\cdot)$, i.e.,
\begin{align*}
N_\Theta:=L_D\circ \sigma\circ L_{D-1}\circ\sigma\circ\cdots\circ\sigma\circ L_{1},
\end{align*}
where
\begin{align*}
L_1(x)&:=W_1x+b_1,\quad W_1\in\mathbb{R}^{d_1\times d},b_1\in\mathbb{R}^{d_1},\\
L_i(x)&:=W_ix+b_i,\quad W_i\in\mathbb{R}^{d_i\times d_{i-1}},b_i\in\mathbb{R}^{d_i} ,\forall i=2,3,\cdots D-1,\\
L_D(x)&:=W_Dx+b_D,\quad W_D\in\mathbb{R}^{N\times d_{D-1}},b_{D}\in\mathbb{R}^{N}.
\end{align*}
Then the parameter collection is $\Theta:=\{W_1,b_1, W_2, b_2, \cdots, W_D, b_D\}$. $N_{\Theta}$ is the trainable part in our networks. Besides the part, we define the prior \textit{dictionary}  $D:\mathbb{R}^d\rightarrow\mathbb{R}^N$ as a vector-valued function, i.e.,
\begin{align*}
D(x):=[f_1(x), f_2(x),\cdots, f_N(x)]\in\mathbb{R}^N,
\end{align*}
where $f_i:\mathbb{R}^d\rightarrow\mathbb{R}, i=1,2,\cdots, N$ are called \textit{word} functions of dictionary $D$. Thus prior information is encoded in these word functions.
Combining the trainable part and the given prior, we formulate  a PD-PINN as
\begin{align}\label{inner_product}
F_\Theta(x):=\langle D(x), N_\Theta(x)\rangle, \quad x\in\mathbb{R}^d,
\end{align}
the structure of which has several advantages:
\begin{itemize}
	\item{\it Plug and play}. Prior dictionaries are not integrated into the essential trainable neural networks so that there is no need to design special networks for learning the priors to solve various problems. Instead, only a designed dictionary in the fusion step should be updated. 
	\item{\it Interpretation}. Physics informed priors are fused with the uninterpretable network via the simple inner product operation, which falls into the area of generalized linear models, and the linear form can usually provide physical significance. For example, if we make the dictionary be a family of trigonometric functions $\{\cos(kx), \sin(kx)\}$, one may interpret the $k$-th element of $N_{\Theta}(x)$ as the magnitude of certain frequency at position $x$. 
	\item{\it Flexible Prior selection.} Since the dictionary and the essential neural network are independent before the final fusion, there is no restriction on the choice of dictionaries. A variety of word functions are available and can be flexibly selected for specific problems.
\end{itemize}
We can construct a dictionary based on the following considerations:
\begin{enumerate}
	\item \textbf{Spatial-based dictionaries.} This kind of dictionaries is designed to encode local magnitudes in word functions. For example, to solve equations on $\Omega$ with a support $\hat \Omega\subset \Omega$, we may construct word functions with supports on $\hat \Omega$.
	
	\item \textbf{Frequency-based dictionaries.} These dictionaries embed frequency priors into word functions such that neural networks enjoy the representation ability in both spatial and frequency domains. 
	Since convergence in frequency domains seems vital in training neural networks  \cite{xu2019training,luo2019theory}, the frequency-based dictionaries may accelerate training stages, especially for the periodic ground truth functions. Our numerical simulation employs this kind of dictionaries. We consider 1d Fourier basis in section  \ref{poisson1d} and \ref{poissonTime}.  Two-dimensional Fourier basis is considered in section \ref{poisson2d}. Sphere harmonic basis is employed in section \ref{poissonsph}.
	
	\item \textbf{Orthogonality. }There is no mandatory orthogonal requirements on  dictionaries.  However, since we have not included any normalization techniques yet, dictionaries with orthogonal word functions are employed in our simulations for stability.
	
	\item \textbf{Learnable dictionaries.} Instead of assigning word functions manually, dictionary construction can also be driven by data. In practice, we may be required to solve the same equation several times with varying boundary value conditions. These solutions may share some common features and could be learned by Principle Component Analysis(PCA)\cite{mohammadi2014pca}, Nonnegative Matrix Factorization(NMF)\cite{esser2012convex, berry2007algorithms} and other dictionary learning techniques\cite{mairal2011task}.
\end{enumerate}

\subsection{Error Bounds of PINNs}\label{pinn-thm}
In this subsection, we provide an error bound on the discrepancy between a trained $F_{\theta}$ and the ground truth under mild assumptions. Consider equation \eqref{pde} with the second order operator:
\begin{align*}
\mathcal{L}[u]:=\sum_{i,j}a_{i,j}(x)\frac{\partial^2 u }{\partial x_i x_j}+\sum_{i}b_i(x)\frac{\partial u}{\partial x_i}+c(x)u.
\end{align*}
We denote $A(x):=[a_{i,j}(x)]_{i,j}\in\mathbb{R}^{d\times d}$ and $b(x):=[b_i(x)]_{i}\in\mathbb{R}^d$. If $\zeta^TA(x)\zeta\geq \lambda(x)\|\zeta\|^2$ holds for some function $\lambda(x)>0$ on $\Omega$, we say $\mathcal{L}$ is strictly elliptic on $\Omega$. In the following theorem, for simplicity, we suppose that $\lambda_0>0$ is a uniformly lower bound of $\lambda(x)$,  $\|b(x)\|_{\infty}$ is upper bounded, and $c(x)\leq 0$ over $\bar\Omega$, where $\bar \Omega$ represents the closure of $\Omega$. Please refer to \cite{gilbarg2015elliptic} for  explicit explanation of the symbols used in this subsection.

\begin{theorem}[Error bounds of PINNs on elliptic PDEs]\label{thm1}
	Suppose that $\Omega\subset \mathbb{R}^d$ is a bounded domain,  $\mathcal{L}$ is strictly elliptic and $\tilde u\in C^0(\bar \Omega)\cap C^2(\Omega)$ is a solution to \eqref{pde}. If the neural network $F_\Theta$ satisfies that
	\begin{enumerate}
		\item[\rm{(1).}] $ \sup_{x\in\partial\Omega}|F_{\Theta}(x)-\tilde u(x)|<\delta_1$;
		\item[\rm{(2).}]  $\sup_{x\in\Omega} |\mathcal{L}[F_\Theta](x)-q(x)|< \delta_2$;
		\item[\rm{(3).}] $F_\Theta\in C^0(\bar \Omega)\cap C^2(\Omega)$,
	\end{enumerate}
	then the error of $F_\Theta(\cdot)$ over $\Omega$ is bounded by
	\begin{align}
	\sup _{x\in\Omega}|F_{\Theta}(x)-u(x)| \leqslant \delta_1+C \frac{\delta_2}{\lambda_0}, \label{thm1_eq}
	\end{align}
	where $C$ is a positive constant depending on $\mathcal{L}$ and  $\Omega$.
\end{theorem}
\begin{proof}
	Denote $h_1 := \mathcal{L}[F_\Theta]-q$ and $h_2:=F_\Theta-\tilde u$. Since $F_\Theta$ and $\tilde u$ fall in  $C^0(\bar\Omega)\cap C^2(\Omega)$, then we have $h_2\in C^0(\bar\Omega)\cap C^2(\Omega)$.  Thanks to Theorem 3.7 in \cite{gilbarg2015elliptic}, we obtain
	\begin{align*}
	\sup_\Omega|h_2(x)|\leq\sup_{\partial \Omega}|h_2(x)|+C\cdot \sup_\Omega\frac{|h_1(x)|}{\lambda(x)},
	\end{align*}
	where $C$ is a positive constant depending only on  $\sup_\Omega\{\|b(x)/\lambda(x)\|\}$ and the diameter of $\Omega$. It immediately follows that inequality \eqref{thm1_eq} holds.
\end{proof}
For Poisson's equations, the second order operator $\mathcal{L}$ degenerates into the Laplace operator $\Delta$, where $a_{ij}(x)=\delta_{ij}$, $b_i(x)\equiv 0$ , $c(x)\equiv 0$ and $\lambda_0=1$. Thus we have the following corollary:
\begin{corollary}\label{poisson_thm}
	Suppose that $\Omega\subset \mathbb{R}^d$ is a bounded domain, and the ground truth $\tilde u\in C^0(\bar\Omega)\cap C^2(\Omega)$. If a neural network $F_\Theta$ satisfies that
	\begin{enumerate}
		\item[\rm{(1).}] $ \sup_{x\in\partial\Omega}|F_{\Theta}(x)-\tilde u(x)|<\delta_1$;
		\item[\rm{(2).}]  $\sup_{x\in\Omega} |\Delta F_\Theta(x)-q(x)|< \delta_2$;
		\item[\rm{(3).}] $F_\Theta\in C^0(\bar \Omega)\cap C^2(\Omega)$;
		\item[\rm{(4).}] $\Omega$ lies between two parallel planes a distance $d$ apart,
	\end{enumerate}
	then the error of $F_\theta(\cdot)$ over $\Omega$ is bounded by
	\begin{align*}
	\sup _{x\in\Omega}|F_{\Theta}(x)-\tilde u(x)| \leqslant \left(\delta_1+\left(e^{d}-1\right) \delta_2\right).
	\end{align*}
\end{corollary}
\begin{proof}
	The proof is similar to Corollary 3.8 in \cite{gilbarg2015elliptic}, and we omit it.
\end{proof}
We discuss the assumptions in Theorem 2.1 and Corollary 2.2. 
\begin{itemize}
	\item If we regard $F_\Theta$ as an input-output black box, it seems impossible to verify whether $F_\Theta$ satisfies condition (1) and (2). In practice, we can sample sufficient points in $\Omega$ and $\partial\Omega$ to estimate the expected loss. Then the expected loss seems more reasonable than $\sup(\cdot)$. In Theorem \ref{l2thm}, we will give an error bound under the expectation sense instead of $\sup(\cdot)$ .
	
	\item Conditions (1) and (2) imply that PINNs can solve elliptic PDEs stably with noises. Suppose that $\tilde u=u_{true} +\varepsilon_1$ and $q=q_{true}+\varepsilon_2$. The error bound \eqref{thm1_eq} then becomes 
	\begin{align}
	\sup _{x\in\Omega}|F_{\Theta}(x)-u_{true}(x)| \leqslant \delta_1+\sup_{\partial \Omega}\varepsilon_1(x)+\frac{C}{\lambda_0} (\delta_2+\sup_{\Omega}\varepsilon_2(x)) \label{thm1_eq_noise}.
	\end{align}
	\item Condition (4) in Theorem \ref{poisson_thm} implies that a narrow region $\Omega$ may reduce errors of PINNs. 
\end{itemize}
In the following, we will measure the discrepancy via an expected loss function instead of $\sup(\cdot)$ to derive the error bound.
To achieve the goal, we choose a smooth dictionary  and smooth activation functions such that $F_\Theta\in C^\infty(\bar\Omega)$. Then we obtain that $\mathcal{L}[F_\Theta]$ is  $l_0$-Lipschitz continuous on $\bar\Omega$ for some constant $l_0>0$. We additionally assume that $q$ is also $l_0$-Lipschitz continuous. Before we propose the final theorem, we construct a relationship between $\|\cdot\|_{L_1}$ and $\sup(\cdot)$ in the following lemma:
\begin{lemma}\label{lemma1} Let $\Omega\subset\mathbb{R}^d$ be a domain. Define the regularity of $\Omega$ as
	\begin{align*}
	R_\Omega:=\inf_{x\in\Omega,r>0}\frac{|B(x,r)\cap\Omega|}{\min\left\{|\Omega|,\frac{\pi^{d/2}r^d}{\Gamma(d/2+1)}\right\}},
	\end{align*}    
	where $B(x,r):=\{y\in\mathbb{R}^d| \|y-x\|\leq r\}$ and $|S|$ is the Lebesgue measure of a set $S$.
	Suppose that $\Omega$ is bounded and $R_\Omega>0$. Let
	$f\in C^0(\bar \Omega)$ be an $l_0$-Lipschitz continuous function on $\bar\Omega$. Then
	\begin{align}\label{lemma_ineq}
	\sup_{\Omega} |f|\leq \max\left\{\frac{2\|f\|_{L_1}}{R_{\Omega}|\Omega|}, 2l_0\cdot\left(\frac{\|f\|_{L_1}\cdot\Gamma(d/2+1)}{l_0R_{\Omega}\cdot \pi^{d/2}}\right)^{\frac{1}{d+1}}\right\}.
	\end{align}
\end{lemma}
\begin{proof}
	According to the definition of $l$-Lipschitz continuity, we have
	\begin{align*}
	l \|x-\bar x\|\geq|f(x)-f(\bar x)|,\quad \forall x,\bar x\in\bar \Omega,
	\end{align*}
	which follows
	\begin{align}\label{ineq1}
	\|f\|_{L_1(\bar \Omega)}\geq\int_{\Omega^+} |f(x)|dx\geq \int_{\Omega^+} |f(\bar x)|-l\|\bar x-x\|dx,
	\end{align}
	where $\Omega^+:=\{x\in\bar \Omega| |f(\bar x)|-l\|\bar x-x\|\geq 0\}$.
	Without loss of generality, we assume that $\bar x\in\arg\max_{\bar\Omega} |f|$ and $f(\bar x)>0$.
	Denote that
	\begin{align*}
	B_1:=B\left(\bar x,\frac{f(\bar x)}{2l}\right)\cap\Omega.
	\end{align*}
	It obvious that $B_1\subset\Omega^+$. Note that the Lebesgue measure of a hypersphere in $\mathbb{R}^d$ with radius $r$ is ${\pi^{d/2}r^d}/{\Gamma(d/2+1)}
	$. Then \eqref{ineq1} becomes
	\begin{align*}
	\|f\|_{L_1}&\geq\int_{B_1} f(\bar x)-l\|\bar x-x\|dx\nonumber\\
	&\geq |B_1|\cdot \frac{f(\bar x)}{2}\nonumber\\
	&\geq\frac{f(\bar x)}{2} \cdot R_\Omega\cdot\min\left\{|\Omega|,\frac{\pi^{d/2}{{f(\bar x)}}^d}{2^d l^d\Gamma(d/2+1)}\right\},
	\end{align*}
	which leads to \eqref{lemma_ineq}.
\end{proof}
It is obvious that $R_\Omega\leq1$ always holds. For various of domains $\Omega$ in practice, we have $R_\Omega>0$. For example, a square domain $\Omega\subset \mathbb{R}^3$ has $R_\Omega=1/8$. For a circle domain $\Omega\subset \mathbb{R}^2$, the regularity $R_\Omega$ is lower bounded by $2/3-\sqrt{3}/(2\pi)\approx 0.391$.

If we adopt smooth activation functions such as $\tanh(\cdot)$ and sigmoid, derivatives of neural networks are also smooth on $\bar \Omega$. Therefore, the Lipschitz continuity of $\mathcal{L}[F_\Theta]$ could be guaranteed.  Further analysis and estimation of the Lipschitz property of neural networks can be found in \cite{virmaux2018lipschitz}.
\begin{theorem}[Error bounds of PINNs on elliptic PDEs]\label{l2thm}
	Suppose that $\Omega\subset \mathbb{R}^d$ is a bounded domain, $\mathcal{L}$ is strictly elliptic and $\tilde u\in C^0(\bar \Omega)\cap C^2(\Omega)$ is a solution to \eqref{pde}. If neural network $F_\Theta$ satisfies that
	\begin{enumerate}
		\item[\rm{(1).}] $ \mathbb{E}_{\mathbf{Y}}\left[|F_{\Theta}(\mathbf{Y})-\tilde u(\mathbf{Y})|\right]<\delta_1$ where the random variable $\mathbf{Y}$ is uniformly distributed on $\partial\Omega$,
		\item[\rm{(2).}]  $\mathbb{E}_{\mathbf{X}}\left[ |\mathcal{L}[F_\Theta](\mathbf{X})-q(\mathbf{X})|\right]< \delta_2$ where the random variable $\mathbf{X}$ is uniformly distributed on $\Omega$,
		\item[\rm{(3).}] $F_\Theta, \mathcal{L}[F_\Theta], \tilde u, q$ are $\frac{l}{2}$-Lipschitz continuous on $\bar\Omega$,
	\end{enumerate}
	then the error of $F_\Theta(\cdot)$ over $\Omega$ is bounded by
	\begin{align}\label{theorem2_ineq}
	\sup _{x\in\Omega}\left| F_{\Theta}(x)-u(x) \right| \leqslant \tilde\delta_1+C \frac{\tilde\delta_2}{\lambda_0}, 
	\end{align}
	where $C$ is a positive constant depending on $\mathcal{L}$ and  $\Omega$, 
	\begin{align*}
	\tilde \delta_1:=& \max\left\{\frac{2\delta_1}{R_{\partial\Omega}}, 2l\cdot\left(\frac{\delta_1|\partial\Omega|\cdot\Gamma(d/2+1)}{lR_{\partial\Omega}\cdot \pi^{d/2}}\right)^{\frac{1}{d+1}}\right\},\quad\text{and}\\
	\tilde \delta_1:=& \max\left\{\frac{2\delta_1}{R_{\Omega}}, 2l\cdot\left(\frac{\delta_1|\Omega|\cdot\Gamma(d/2+1)}{lR_{\Omega}\cdot \pi^{d/2}}\right)^{\frac{1}{d+1}}\right\}.
	\end{align*}
	\begin{proof}
		Note that $\mathbb{E}_{\mathbf{Y}}[\cdot]=\|\cdot\|_{L_1(\partial\Omega)}/|\partial\Omega|$ and $\mathbb{E}_{\mathbf{X}}[\cdot]=\|\cdot\|_{L_1(\Omega)}/|\Omega|$ . $F_\Theta-\tilde u$ and $\mathcal{L} [F_\Theta] -q$ are both $l$-Lipschitz continuous on $\bar \Omega$. Combining Theorem \ref{thm1} and Lemma \ref{lemma1} leads to  \eqref{theorem2_ineq}.
	\end{proof}
\end{theorem}
In Theorem \ref{l2thm}, we have proved that  when the tractable training loss decreases to $0$ in the sense of expectation, the neural network $F_\Theta$ approximates the ground truth.
\section{Numerical Experiments}
Our implementation is heavily inspired by the framework DeepXDE\footnote{\url{https://deepxde.readthedocs.io/}}. We also take the following standard technical settings in the numerical simulations:
\begin{itemize}
	\item \textbf{Initialization.} The initialization of $\Theta$ might be vital, but this topic goes beyond our discussion. Instead, we employ the standard initialization. Each entry in $W_i$ and $b_i$ is uniformly and independently distributed on the interval $[-1/\sqrt{d_{i-1}},1/\sqrt{d_{i-1}}]$.
	
	\item \textbf{Sampling.} We focus on the comparison between PINNs\cite{lu2019deepxde} and PD-PINNs in this section, where uniform distribution is applied to all simulations. Indeed, one can employ more sophisticated  sampling strategies with adaptive adjustment\cite{lu2019deepxde} in the training of both PINNs and PD-PINNs.
	
	\item \textbf{Optimizer.} The popular optimizer Adam\cite{KingmaB14} with the learning rate of $0.001$ is employed in this section.
	\item \textbf{Loss.} 
	Since the multi layer neural networks are regarded as  black-boxes in context, computing analytic forms of losses seems intractable. As illustrated in \eqref{pde_loss_1} and \eqref{bc_loss_1}, in the $k$-th iteration we estimate $Loss_{PDE}$ via the empirical loss function
	\begin{align}\label{pde_loss_2}
	\widehat{Loss}_{PDE}(\Theta_k):=\frac{1}{N_{PDE}}\sum_{i=1}^{N_{PDE}}|\mathcal{L}[F_{\Theta_k}](\mathbf{X}^{k}_i)-q(\mathbf{X}^{k}_i)|^2,
	\end{align}
	where $\mathbf{X}_i^{k}$ are i.i.d. variables uniformly distributed on $\Omega$. We also estimate the $Loss_{BC}(\Theta)$ via
	\begin{align}\label{bc_loss_2}
	\widehat{Loss}_{BC}(\Theta_k):=\frac{1}{N_{BC}}\sum_{i=1}^{N_{BC}}|F_{\Theta_k}(\mathbf{Y}^{k}_i)-\tilde u(\mathbf{Y}^{k}_i)|^2,
	\end{align}
	
	where $\mathbf{Y}_i^{k}$ are i.i.d. variables uniformly distributed on $\partial\Omega$. Since it is also hard to track exact prediction error between $F_\Theta$ and the ground truth
	$\tilde u$ during training stages, in the $k$-th iteration we  employ a Monte Carlo way to estimate $\|F_\Theta-\tilde u\|_{L_2(\Omega)}^2$ via
	\begin{align*}
	\widehat{Error}_{predict}(\Theta_k):=\frac{1}{N_{pred}}\sum_{i=1}^{N_{pred}} |F_{\Theta_k}(\mathbf{Z}^{k}_i)-\tilde u(\mathbf{Z}^{k}_i)|^2,
	\end{align*}
	where $\mathbf{Z}_i^{k}$ are i.i.d. variables uniformly distributed on $\Omega$. We set $N_{pred}=1000$ in this section.
\end{itemize}
All simulations in this section are conducted with PyTorch.  
The code to reproduce all the results is available online\footnote{\url{https://github.com/weipengOO98/PDPINN.git}}.
\subsection{1d Poisson's Equation}\label{poisson1d}
First, we consider a one dimensional Poisson's equation with the Dirichlet boundary condition on both ends. Though the 1d problem seems simpler than its higher dimension versions, it is actually hard  for neural networks to learn. The value of $F_\Theta$ at an interior point is decided by two paths which connect the interior point and the two boundary ends. A slightly large error on one of the paths will result in large error in predictions of interior values.

Consider the ground truth:
\begin{align*}
\hat u:=\sin(0.7x)+\cos(1.5x)-0.1x, \quad \forall x\in[-10,10],
\end{align*}
which is smooth and has two different frequency components combining with a linear term. Its graph is shown in Figure \ref{f1}.
\begin{figure}[htbp]
	\centering
	\includegraphics[width=0.4\textwidth]{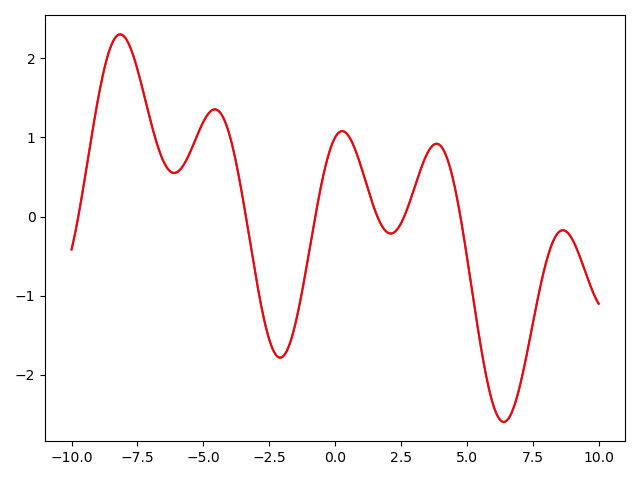}
	\caption{Illustration of $\hat u$.}
	\label{f1}
\end{figure}
The corresponding 1d Poisson's equation is formulated as follows:
\begin{align*}
\frac{d^2u}{dx^2}&=-0.49\cdot \sin(0.7  x) - 2.25\cdot \cos(1.5 x),\\
u(-10)&=\hat u(-10),\\
u(10)&=\hat u(10).
\end{align*}
We employ a frequency based dictionary $D$ with $2k+1$ word functions:
\begin{align*}
D(x)=\left\{1,\cos(x),\sin(x), \cos(2x),\sin(2x),\cdots,\cos(kx),\sin(kx)\right\}.
\end{align*}
Take $N_{PDE}=100$, and the boundary value condition at two ends is included in the loss function in each iteration. The results are shown in Figure \ref{poisson2d}. PINNs implemented by MLPs fail to find the ground truth, though the curvature shares some similar tendency with $\hat u$. The failure might be caused by the propagation perturbation of boundary information. However, with  dictionary $D$ integrated, the PD-PINNs have the ability to represent higher frequency even at initial iterations, and this ability might allow $F_\Theta$ to broadcast information via the frequency domain instantly instead of gradual transmission through the spatial domain. 

\begin{figure}[htbp]
	\centering
	\includegraphics[width=1\textwidth]{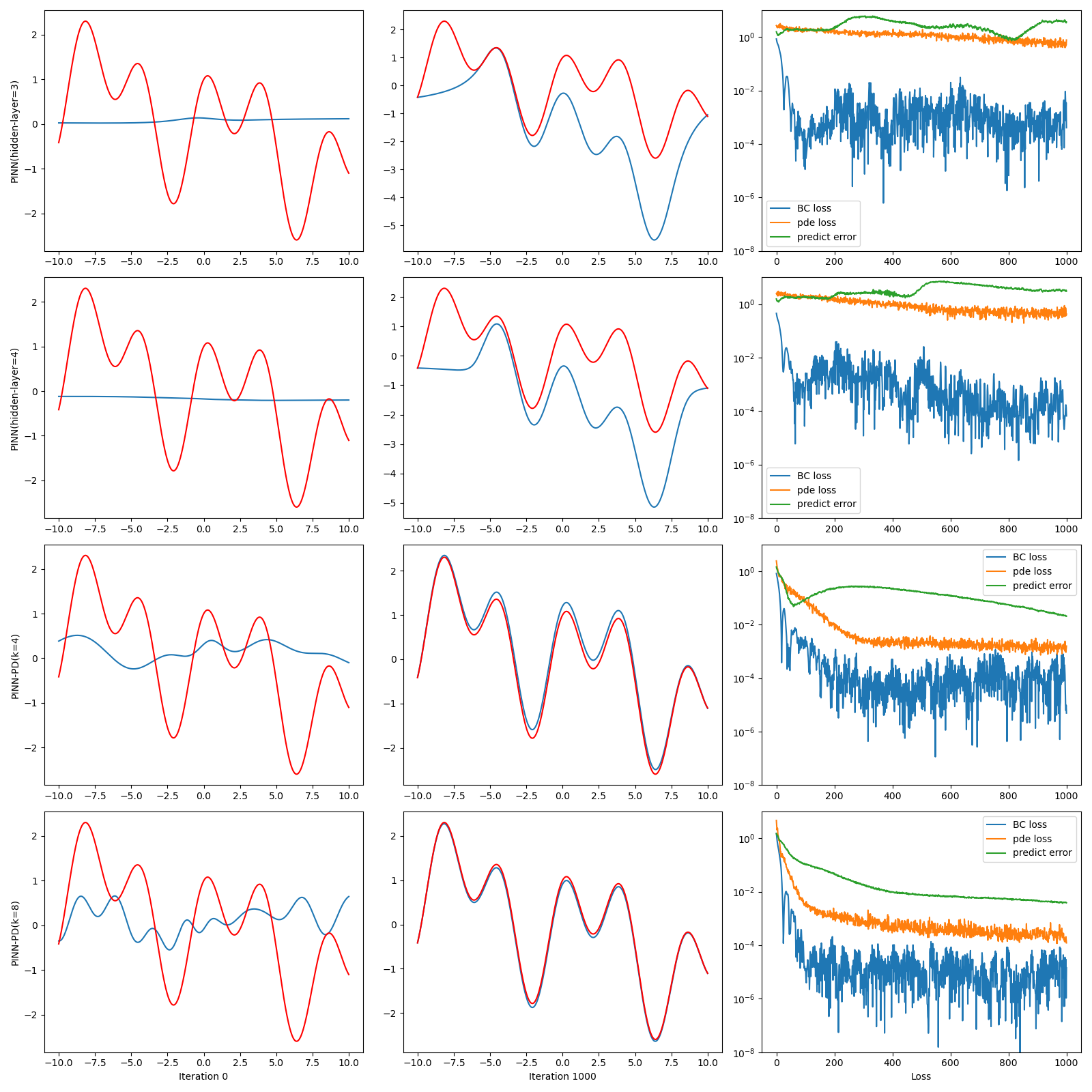}
	\caption{The first row shows the result of an MLP with 3 hidden layers. The second row displays the result of an MLP with 4 hidden layers. The third row employs a PD-PINN with 3 hidden layers and $k=4$. The fourth row employs a PD-PINN with 3 hidden layers and $k=8$. With the ground truth marked in red, the first column displays initial response curves while the second column shows response curves of $F_\Theta$ after 1000 iterations.}
	\label{p1_all}
\end{figure}
\subsection{2d Poisson's Equation}\label{poisson2d}
Define the ground truth on $[-10,10]\times[-10,10]\subset\mathbb{R}^2$:
\begin{align*}
\hat u(x,y)=\left(\sin (0.7x) +\cos (1.5x) - 0.1 x\right)\cdot\sin\left(
\frac{y + 10}{20}\pi\right).
\end{align*}
The graph of $\hat u$ is shown in Figure \ref{f5}.
\begin{figure}[htbp]
	\centering
	\includegraphics[width=0.7\textwidth]{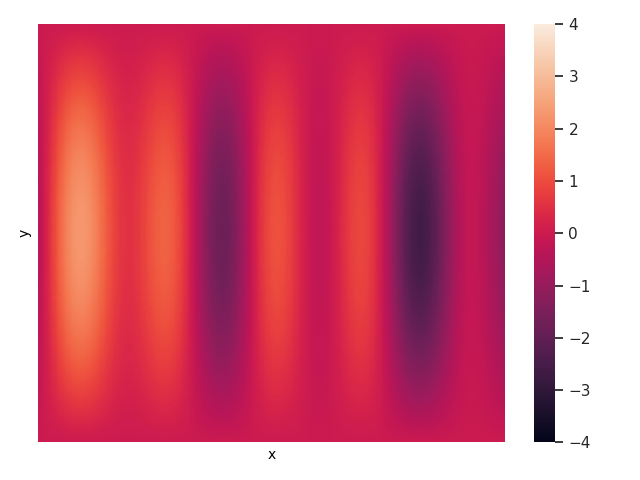}
	\caption{Illustration of $\hat u$}
	\label{f5}
\end{figure}
We formulate the 2d Poisson's equation as
\begin{align*}
u_{xx}+u_{yy}=&-\sin\left(\frac{y+ 10}{ 20} \pi\right)
(0.49 \sin(0.7 x) + 2.25 \cos(1.5 x))\\
&- (\sin(0.7 x) + \cos(1.5 x) - 0.1 x) \sin\left(\frac{y+ 10}{20}\pi\right) \frac{\pi^2} {400},\\
u(x, 10)=&0,\quad u(x,-10)=0, \quad x\in[-10,10],\\
u(10,y)=&\hat u(10,y),\quad y\in[-10,10],\\
u(-10,y)=&\hat u(-10,y), \quad y\in[-10,10].\\
\end{align*}
We construct a dictionary $D_{k_1,k_2}$ via
\begin{align*}
D_{k_1}^1&:=\left\{1, \sin(\pi x),\frac{\sin(2\pi x)}{2},\cdots, \frac{\sin((k_1-1)\pi x)}{k_1-1} \right\},\\
D_{k_2}^2&:=\left\{1, \sin(\pi y),\frac{\sin(2\pi y)}{2},\cdots, \frac{\sin((k_2-1)\pi y)}{k_2-1} \right\},\\
D_{k_1,k_2}&:=\left\{f_1f_2| f_1\in D_{k_1}^1, f_2\in D_{k_2}^2 \right\}.
\end{align*}
Take $N_{PDE}=1000$ and $N_{BC}=400$. Setting $k_1=k_2=5$, we have $25$ word functions in this dictionary. The result is shown in Figure \ref{diffusion_all}. It is obvious that the PD-PINN outperforms the PINN on this problem.
\begin{figure}[htbp]
	\centering
	\includegraphics[width=1\textwidth]{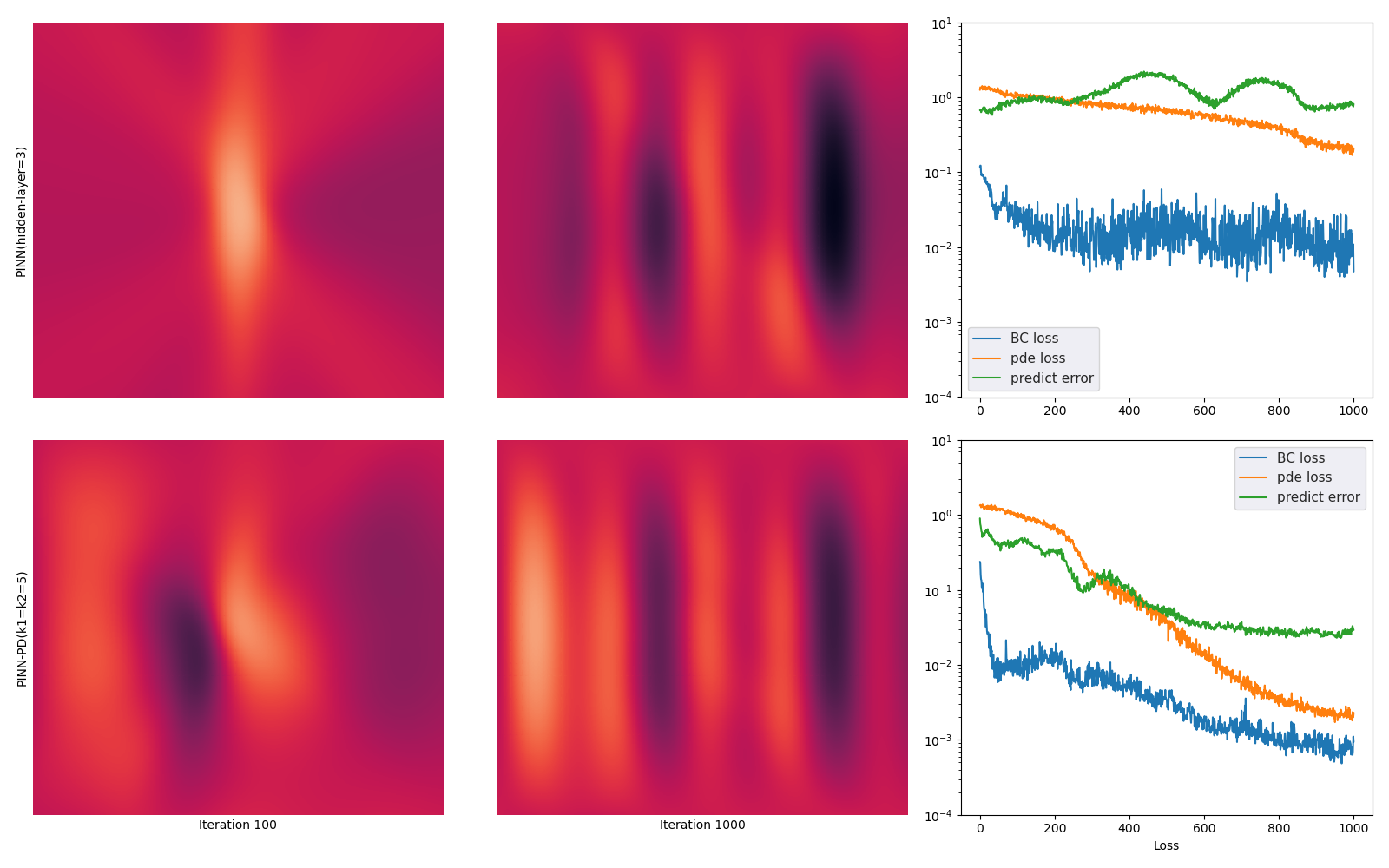}
	\caption{The first row shows the result of an MLP with 4 hidden layers. The second row employs a PD-PINN with 3 hidden layers and dictionary  $D_{5,5}$. The first column displays the response surface after 100 iterations while the second column shows that of $F_\Theta$ after 1000 iterations. }
	\label{diffusion_all}
\end{figure}

\subsection{Spherical Poisson's Equation}\label{poissonsph}
We consider the solution of Poisson's equation on a sphere and PD-PINNs with the sphere Harmonic basis as a dictionary. 

Let $u(\theta, \phi)$ be a scalar function on a sphere, where the location of a point is indicated by colatitude $0\leq\theta\leq\pi$ and longitude $0\leq\phi<2\pi$.
We employs the special form\cite{yee1981solution} in the experiment. Let the ground truth be 
\begin{align*}
\hat u(\theta,\phi):=\cos\theta \cdot \sin^M\theta\cdot  \cos(
M \phi ) -\cos\theta \cdot \sin^{M-1}\theta\cdot  \cos(
(M-1) \phi ), \quad M=7.
\end{align*}
Its Mercator projection is displayed in Figure \ref{mp}.
\begin{figure}[htbp]
	\centering
	\includegraphics[width=0.9\textwidth]{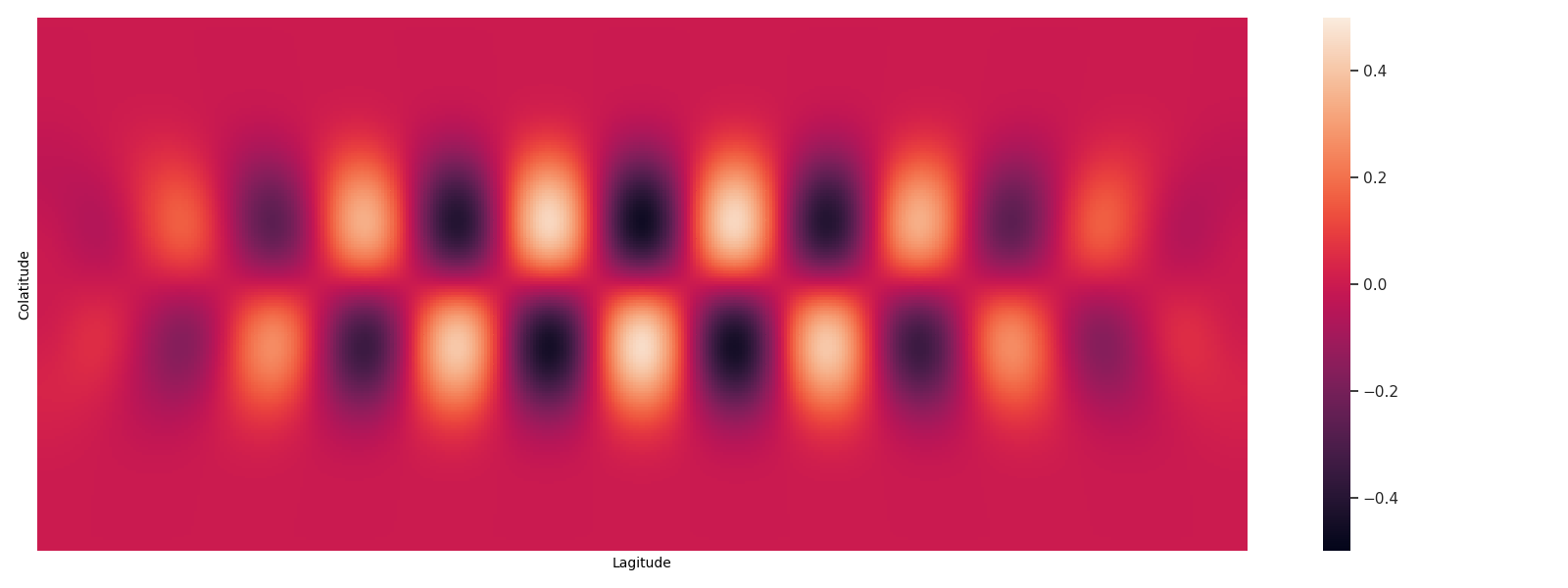}
	\caption{The Mercator projection of $\hat u(\theta,\phi)$}
	\label{mp}
\end{figure}
We formulate the Poisson's equation on the sphere as:
\begin{align}
\frac{1}{\sin \theta} \cdot\frac{\partial}{\partial \theta} \left(\sin \theta \frac{\partial u}{\partial \theta}\right)+\frac{1}{\sin ^{2} \theta}\cdot \frac{\partial^{2} u}{\partial \phi^{2}}&=f(\theta, \phi),\\
u(1,1)&=\hat u(1,1)\label{bc_sph},
\end{align}
where
\begin{align*}
f(\theta, \phi)=&-(M + 1) (M + 2) \cos\theta\sin^M\theta \cos(M \phi ) \\
&+M(M + 1) \cos\theta \sin^{M-1}\theta  \cos(
(M-1) \phi), \quad M=7.
\end{align*}
Note that \eqref{bc_sph} is the boundary value condition, which is a single point but enough to make the solution unique. We also alter the structure of neural networks employed in this subsection. As is shown in Figure \ref{ebd3d}, we put a lifting layer right after the input layer, which lifts $(\theta, \phi)\subset\mathbb{R}^2$ to $(x,y,z)\in\mathbb{R}^3$ via
\begin{align*}
(x,y,z)=(\sin\theta\sin\phi, \sin\theta\cos\phi,\cos \theta).
\end{align*}
\begin{figure}[htbp]
	\centering
	\includegraphics[width=1\textwidth]{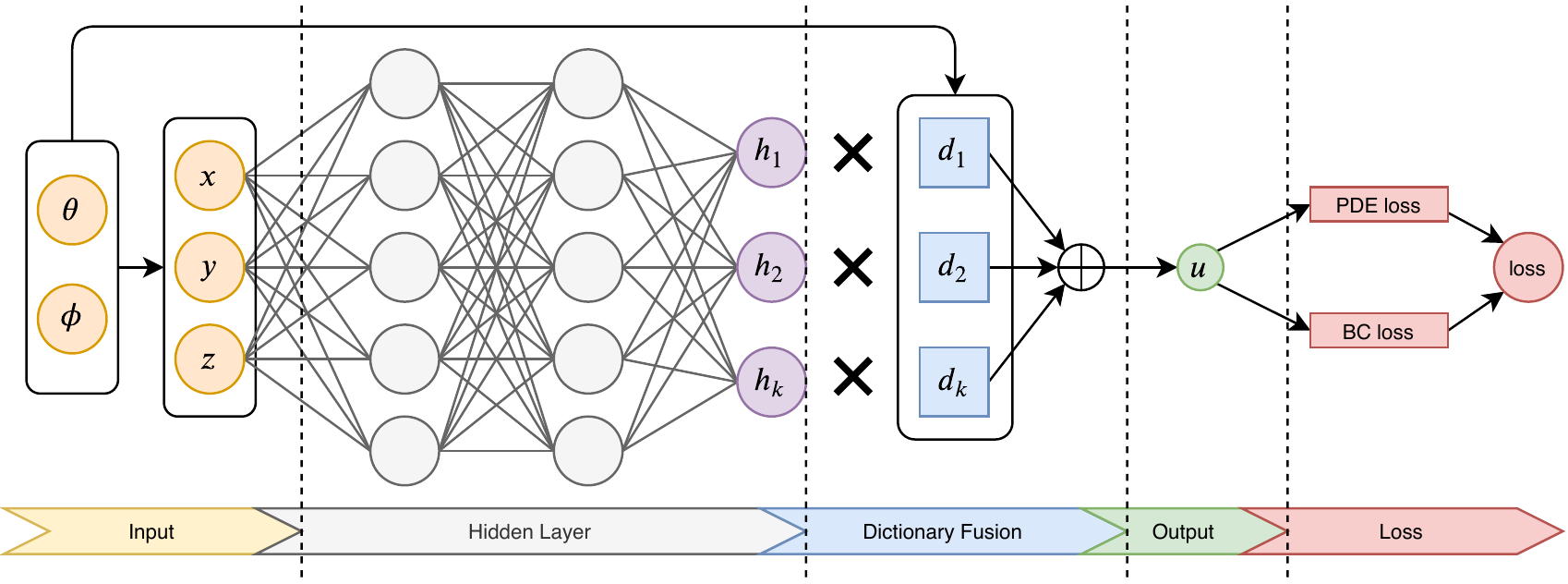}
	\caption{Illustration of PD-PINN with lifting.}
	\label{ebd3d}
\end{figure}
To construct the dictionary, we employ $16$ real spherical harmonic basis functions\cite{maintz2016efficient} as the word functions,
\begin{align*}
D&:=D^+\cup D^-\quad \text{with}\\
D^+&:=\{C_{m,l}\cdot \cos(m\phi)\cdot P_{m}^{l}(\cos(\theta))|0\leq l\leq3, 0\leq m\leq l\}\\
D^-&:=\{C_{m,l}\cdot \sin(m\phi)\cdot P_{m}^{l}(\cos(\theta))|0\leq l\leq3, 0> m\geq-l\},
\end{align*}
where $P_{m}^l(\cdot)$ are the associated Legendre polynomial functions and $C_{m,l}$ are normalization constants.  Set $N_{PDE}=200$, and \eqref{bc_sph} is taken into account in each iteration. The results are shown in Figure \ref{f4}. The PINN fails to recover $\hat u$ in $2000$ iterations while the PD-PINN recovers the ground truth with the error below 0.001.
\begin{figure}[htbp]
	\centering
	\includegraphics[width=0.9\textwidth]{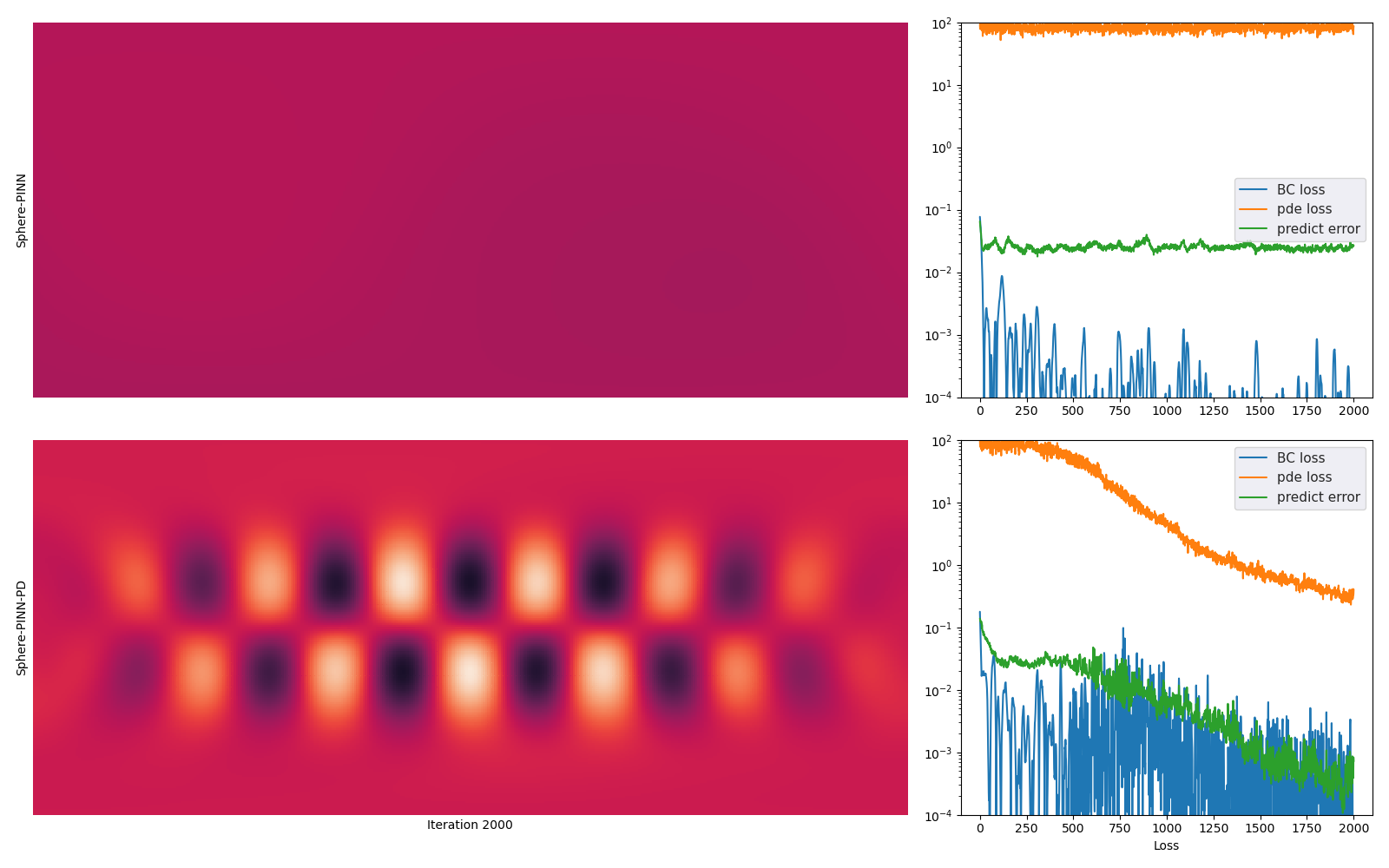}
	\caption{The first row is the result that is produced by MLP with 4 hidden layers. The second row is the result that is produced by the PD-PINN.}
	\label{f4}
\end{figure}

\subsection{Diffusion Equation}\label{poissonTime}
The last simulation is conducted on a parabolic equation.
Define the ground truth
\begin{align*}
\hat u(x,t)=(\sin(0.7 x) + \cos(1.5x) - 0.1 x) \cdot t,
\quad(x,t)\in[-10,10]\times[0,1].
\end{align*}
which is illustrated in Figure \ref{f3}:
\begin{figure}[htbp]
	\centering
	\includegraphics[width=0.9\textwidth]{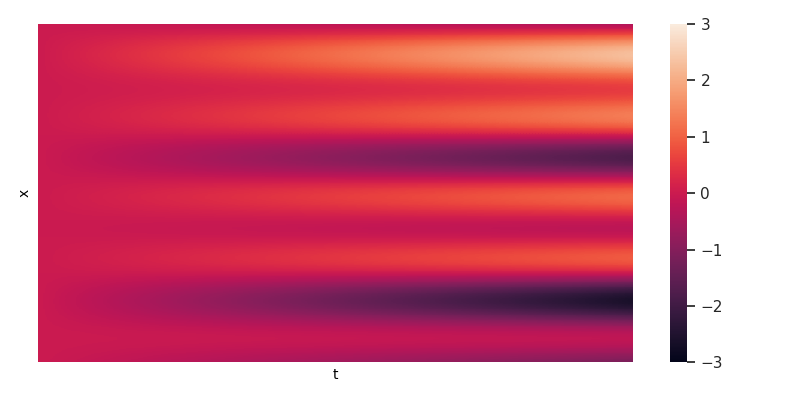}
	\caption{Illustration of $\hat u$}
	\label{f3}
\end{figure}
Consider the one-dimensional diffusion equation:
\begin{align}
u_{xx}(x,t)-u_t(x,t)&=\hat u_{xx}(x,t)-\hat u_t(x,t)\quad,\forall x,t\in[-10,10]\times[0,1],\nonumber\\
u(x,0)&=0,\quad \forall x\in[-10,10],\label{bc1}\\
u(-10,t)&=\hat u(10,t)\equiv 0,\quad \forall t\in[0,1].\label{bc2}
\end{align}
Though the input is two-dimensional, we could employ a dictionary only depends on one of the dimensions:
\begin{align*}
D_k:=\{1,\cos x,\sin x,\cos 2x,\sin 2x ,\cdots, \cos kx, \sin kx\}.
\end{align*}
We employ $D_{10}$ with $21$ words involved. Take $N_{PDE}=1000$ inside. Note that we regard the initial value condition \eqref{bc1} as a boundary value conditions and take $N_{BC}=300$ . As is shown in Figure \ref{diffusion_2d_all}, the PD-PINN outperforms the PINN. As we have emphasized earlier in the manuscript, the loss curve drawn in the last subfigure suggests that the rapid vanishment of $\widehat {Loss}_{BC}$ and $\widehat {Loss}_{PDE}$ do not necessarily imply an equivalent decline of the prediction error.

\begin{figure}[htbp]
	\centering
	\includegraphics[width=0.9\textwidth]{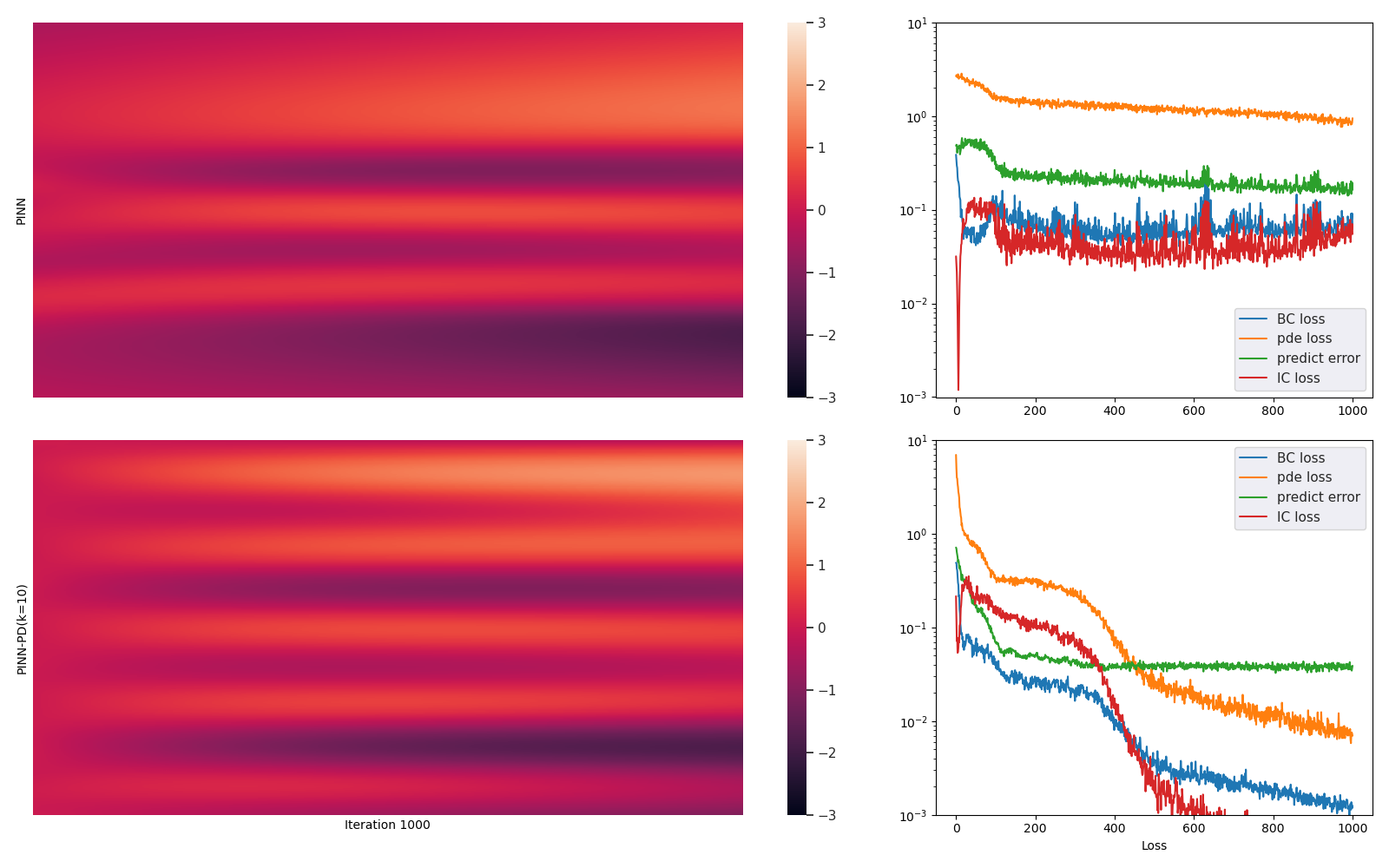}
	\caption{The first row is the result produced by an MLP with 4 hidden layers. The second row is produced by the PD-PINN with $D_{10}$.}
	\label{diffusion_2d_all}
\end{figure}

\section{Conclusion}
In this manuscript, we have proposed a novel PINN structure, which combines PINNs with prior dictionaries. With proper adoption of word functions, we illustrated that PD-PINNs outperform PINNs in our simulations with various settings. We also noted that the convergence of PINNs lacks a theoretical guarantee and thus proposed an error bound on the elliptic PDEs of second order. To our knowledge, this is the first theoretical error analysis on PINNs. 

However, to make PINNs be more practical and universal PDE solvers, we still need to understand the way in which PINNs learn about physics information. Error bounds on other types of PDEs besides elliptic PDEs should also be established.
\section*{Acknowledgements}
We thank Dr. Wenjie Lu and Dr. Dong Cao for their insightful suggestions. This work was supported in part by National Natural Science Foundation of China under Grant No.51675525 and 11725211.

%% The Appendices part is started with the command \appendix;
%% appendix sections are then done as normal sections
%% \appendix

%% \section{}
%% \label{}

%% If you have bibdatabase file and want bibtex to generate the
%% bibitems, please use
%%
%\bibliographystyle{elsarticle-num} 
%\bibliography{ref}

%% else use the following coding to input the bibitems directly in the
%% TeX file.
\end{document}